%% file: conc.tex
\documentclass[11pt,letterpaper]{article}
\input{sections/preamble}
\usepackage[margin=1in]{geometry}

\makeatletter
\def\blfootnote{\gdef\@thefnmark{}\@footnotetext}
\makeatother

\begin{document}
   
   	\title{Concentration of the Langevin Algorithm's Stationary Distribution}
		
	\author{
		Jason M. Altschuler\footnote{JMA was partially supported by an NSF Graduate Research Fellowship, a TwoSigma PhD Fellowship, and an NYU Data Science Faculty Fellowship. 
    }\\
		UPenn \\
		\texttt{alts@upenn.edu}
		\and
		Kunal Talwar \\
		Apple \\
		\texttt{ktalwar@apple.com}
	}
	\date{}
	\maketitle

	\input{sections/abstract}
	\input{sections/intro}
	\input{sections/mgf}
	\input{sections/sc}
	\input{sections/convex}
 	\input{sections/extensions}

	\paragraph*{Acknowledgements.} We thank Sinho Chewi, Murat Erdogdu, and Andre Wibisono for helpful conversations about the related literature. We also thank Sinho Chewi for sharing an early draft of his book~\citep{Chewi22} on log-concave sampling.
	
	\small
	\addcontentsline{toc}{section}{References}
	\bibliographystyle{plainnat}
	\bibliography{sampling_pabi}{}

\end{document}

%% file: sections/preamble.tex
\usepackage{amsmath,amssymb,amsthm}
\usepackage{MnSymbol} 
\usepackage{graphicx,color}
\usepackage[hyphens]{url}
\usepackage{dsfont}
\usepackage{booktabs}
\usepackage[normalem]{ulem}
\usepackage{mathtools}
\usepackage[nameinlink,capitalize]{cleveref}
\usepackage{multirow}
\usepackage{algorithm}
\usepackage[noend]{algpseudocode}
\usepackage{tikz} 

\usepackage{soul} 

\usepackage[square,sort,numbers]{natbib}
\usepackage[sort,nocompress,noadjust]{cite}

\usepackage{subcaption}

\usepackage{tabularx}
\newcolumntype{L}{>{\raggedright\arraybackslash}X}

\numberwithin{equation}{section}
\newtheorem{theorem}{Theorem}[section]

\newtheorem{lemma}[theorem]{Lemma}

\newtheorem{example}[theorem]{Example}

\newtheorem{defin}[theorem]{Definition}

\newcommand{\cN}{\mathcal{N}}

\newcommand{\cR}{\mathcal{R}}

\newcommand{\R}{\mathbb{R}}

\newcommand{\N}{\mathbb{N}}

\newcommand{\Prob}{\mathbb{P}}
\newcommand*{\E}{\mathbb{E}}

\DeclareMathOperator{\logdel}{\log\tfrac{1}{\delta}}

\renewcommand{\leq}{\leqslant}
\renewcommand{\geq}{\geqslant}

\providecommand{\norm}[1]{\lVert{#1}\rVert}

\newcommand{\pieta}{\pi_{\eta}}
\newcommand{\tpieta}{\tilde{\pi}_{\eta}}

\newcommand{\Sd}{\mathcal{S}^{d-1}}

%% file: sections/abstract.tex

\begin{abstract}
	A canonical algorithm for log-concave sampling is the Langevin Algorithm, aka the Langevin Diffusion run with some discretization stepsize $\eta > 0$. This discretization leads the Langevin Algorithm to have a stationary distribution $\pieta$ which differs from the stationary distribution $\pi$ of the Langevin Diffusion, and it is an important challenge to understand whether the well-known properties of $\pi$ extend to $\pieta$. 
	In particular, while concentration properties such as isoperimetry and rapidly decaying tails are classically known for $\pi$, the analogous properties for $\pieta$ are open questions with algorithmic implications. This note provides a first step in this direction by establishing concentration results for $\pieta$ that mirror classical results for $\pi$. Specifically, we show that for any nontrivial stepsize $\eta > 0$, $\pieta$ is sub-exponential (respectively, sub-Gaussian) when the potential is convex (respectively, strongly convex). Moreover, the concentration bounds we show are essentially tight. We also show that these concentration bounds extend to all iterates along the trajectory of the Langevin Algorithm, and to inexact implementations which use sub-Gaussian estimates of the gradient. 
	\par Key to our analysis is the use of a rotation-invariant moment generating function (aka Bessel function) to study the stationary dynamics of the Langevin Algorithm. This technique may be of independent interest because it enables directly analyzing the discrete-time stationary distribution $\pieta$ without going through the continuous-time stationary distribution $\pi$ as an intermediary.
\end{abstract}

%% file: sections/intro.tex

\section{Introduction}\label{sec:intro}

Sampling from a log-concave distribution $\pi \propto e^{-f}$ over $\R^d$ is a foundational problem with applications throughout statistics, engineering, and the sciences.
A canonical and well-studied approach is the Langevin Algorithm, which is the discrete-time Markov process 
\begin{align}
	X_{t+1} = X_t - \eta \nabla f(X_t) + Z_t\,,
	\label{eq:LA}
\end{align}
where $\eta > 0$ is a stepsize parameter, and $Z_t \sim \cN(0, 2\eta I_d)$ are independent Gaussians. Briefly, the intuition behind the Langevin Algorithm is that it discretizes the Langevin Diffusion, which is a continuous-time Markov process with stationary distribution $\pi$. More precisely, the Langevin Diffusion is the stochastic differential equation
\begin{align}
	dX_t = - \nabla f(X_t) + \sqrt{2}d B_t\,,
	\label{eq:LD}
\end{align}
where $B_t$ is a standard Brownian motion on $\R^d$, and the Langevin Algorithm is identical except that it updates the gradient $\nabla f(X_t)$ every $\eta$ units of time rather than continuously. We refer to the textbooks~\citep{Chewi22,andrieu2003introduction,robert1999monte,jerrum1996markov,liu2001monte} for a detailed historical account of the Langevin Algorithm and the truly extensive literature surrounding it, which spans multiple decades and communities.

Although the discretization enables simulating the Langevin Diffusion algorithmically, it introduces an asymptotic bias. Specifically, for any discretization stepsize $\eta > 0$, the stationary distribution $\pieta$ of the Langevin Algorithm differs from the stationary distribution $\pi$ of the Langevin Diffusion. An important question is:
\[
\text{\emph{How similar are }} \pieta \text{\emph{ and }} \pi\text{?}
\]

\paragraph*{Concentration?} 
The motivation of this note is that this question is wide open for concentration properties. 
Indeed, while concentration properties for $\pi$ have been classically known for several decades, the analogous properties for $\pieta$ have remained unknown. Perhaps the most notable example of this is isoperimetry: while classical results imply that $\pi$ satisfies Poincar\'e/log-Sobolev isoperimetric inequalities when the potential $f$ is convex/strongly-convex~\citep{bakry1985diffusions,bobkov2003spectral,kannan1995isoperimetric}, the analogous properties for $\pieta$ are unknown. The influential paper~\citep{VempalaW19} stated this as an open question and proved rapid mixing of the Langevin Algorithm in R\'enyi divergence under this conjectural assumption of isoperimetry of $\pieta$. Another example is rapidly decaying tails: while classical results imply that  $\pi$ has sub-exponential/sub-Gaussian tails when the potential $f$ is convex/strongly-convex~\citep{ledoux1999concentration,MilmanBook}, the analogous properties for $\pieta$ were previously unknown. Such tail decay results can be exploited to extend fast-mixing results for the Langevin Algorithm from constrained settings to unconstrained settings~\citep{AltTal22mix}. 

\paragraph*{Contributions.} The purpose of this note is to provide a first step in this direction. Specifically, in \S\ref{sec:sc} we show that the stationary distribution $\pieta$ of the Langevin Algorithm is sub-Gaussian when the potential $f$ is strongly convex; and in \S\ref{sec:convex} we show that $\pieta$ is sub-exponential when $f$ is convex. These results hold for any nontrivial\footnote{
    Our results apply for $\eta =O(1/M)$. This is the relevant regime in both theory and practice since the Langevin Algorithm is transient if $\eta > 2/M$. In fact, $\eta$ is typically polynomially smaller than $O(1/M)$ to ensure small bias~\citep{Chewi22}.
} discretization stepsize $\eta > 0$ and mirror the aforementioned classical results for $\pi$. Moreover, the concentration bounds we show are essentially tight.
\par In \S\ref{sec:extensions}, we mention two extensions of these results: the concentration bounds extend to all iterates $X_t$ along the trajectory of the Langevin Algorithm, and/or to inexact implementations of the Langevin Algorithm which use sub-Gaussian estimates of the gradient. 

\paragraph*{Techniques.} A central obstacle for proving this result is that this requires establishing properties of $\pieta$ for any discretization stepsize $\eta > 0$, not just $\eta$ arbitrarily small. Since the difference between $\pi$ and $\pieta$ grows with $\eta$, we approach this problem by directly analyzing $\pieta$---rather than analyzing properties of $\pi$ and then trying to conclude properties of $\pieta$ via approximation bounds between $\pi$ and $\pieta$. However, directly analyzing $\pieta$ comes with the difficulty that many standard analysis techniques for the continuous-time Langevin Diffusion become much more complicated after discretization. To overcome this, our analysis makes use of a rotation-invariant moment generating function (aka Bessel function) to directly study the dynamics of the discrete-time Langevin Algorithm at stationarity. Briefly, the key point is that this Lyapunov function readily tracks the effect of gradient descent updates and Gaussian noise convolutions---the two components of the update equation~\eqref{eq:LA} for the Langevin Algorithm---thereby enabling a simple and direct analysis of this discrete-time process. This technique may be of independent interest and is developed in \S\ref{sec:mgf}.

\paragraph*{Independent work.} After completing an initial draft of this manuscript, we found out that a recent revision of~\citep{VempalaW19} (arXiv v4) has a new result (Theorem 8) which shows that in the strongly convex setting, $\pi$ satisfies LSI and thus is sub-Gaussian with concentration parameters that exactly match our Theorem~\ref{thm:sc}. 
The proof techniques are completely different, and the final results are incomparable in their level of generality. On one hand, their result subsumes Theorem~\ref{thm:sc} in that it shows LSI, which implies sub-Gaussianity. On the other hand, our techniques are more general in that they readily extend to the (non-strongly) convex setting, while the techniques of~\citep{VempalaW19} only apply to the strongly convex setting (it is stated as an open question in their paper whether one can obtain the analogous results for this more general setting). We are grateful to Andre Wibisono for making us aware of this recent revision and for discussing the differences with us.

\section{Preliminaries}\label{lem:prelim}

\paragraph*{Notation.}  We write $\Sd$ to denote the unit sphere in $\R^d$, and we abuse notation slightly by writing $v \sim \Sd$ to indicate that $v$ is a random vector drawn uniformly from $\Sd$. Throughout, $\|\cdot\|$ denotes the Euclidean norm. Recall that a differentiable function $f : \R^d \to \R$ is said to be $M$-smooth if $\nabla f$ is $M$-Lipschitz; and is said to be $m$-strongly convex if $f(y) \geq f(x) + \langle \nabla f(x), y-x \rangle + \frac{m}{2}\|x-y\|^2$ for all $x,y$. All logarithms are natural. All other notation is introduced in the main text.

\paragraph*{Integrability assumption.} Throughout, we assume that the convex potential $f : \R^d \to \R$ satisfies $\int e^{-f(x)} dx < \infty$ since this suffices to ensure that the stationary distribution of the Langevin Diffusion exists and is equal to $\pi \propto e^{-f}$~\citep{roberts1996exponential,bhattacharya1978criteria}. This assumption is automatically satisfied in the strongly convex setting, and is a mild assumption in the convex setting (e.g., it holds if there is a minimizer $x^*$ of $f$ and a large enough radius $R$ such that $f(x) > f(x^*)$ at all points $x$ satisfying $\|x-x^*\| = R$.)

%% file: sections/mgf.tex

\section{Lyapunov function}\label{sec:mgf}

We propose to use the following Lyapunov function in order to analyze $\pieta$. 

\begin{defin}[Lyapunov function]
	For any dimension $d \in \N$ and weight $\lambda > 0$, let $\Phi_{d,\lambda} : \R^d \to \R$ denote the function
	\begin{align}
	\Phi_{d,\lambda}(x) := \E_{v \sim \Sd} \left[ e^{\lambda \langle v, x \rangle} \right]\,.
    \label{eq:Phi-def}
	\end{align}
\end{defin}

This definition extends to random variables $X$ by taking the expectation $\Phi_{d,\lambda}(X) = \E_{x} [\Phi_{d,\lambda}(x)]$. The rest of this section discusses properties and interpretations of this Lyapunov function.

\subsection{Relation to rotation-invariant MGF}\label{ssec:mgf}

This Lyapunov function has a natural interpretation in terms of a rotationally-invariant version of the standard moment generating function (MGF). Specifically, rather than bounding the standard MGF of a distribution $\mu$---which we recall is
$\E_{X \sim \mu} e^{\lambda \langle v, X \rangle}\,,$
as a function of a fixed vector $v \in \Sd$---we will bound the average
$	\E_{X \sim \mu} 	\Phi_{d,\lambda}(X) =  \E_{v \sim \Sd} \E_{X \sim \mu} e^{\lambda \langle v, X \rangle}$. The latter is precisely the average of the standard MGF over all vectors $v \in \Sd$, hence why we refer to it as the ``rotation-invariant MGF''.

\par How do bounds on this rotation-invariant MGF imply concentration inequalities? Since properties like sub-Gaussianity and sub-exponentiality of a distribution are by definition equivalent to MGF bounds which hold uniformly over all vectors $v \in \Sd$~\citep{rigollet2015high}, these properties of course imply the same bounds for the rotationally-invariant MGF (simply average over $v \sim \Sd$). The converse is in general lossy, in that a bound on the rotationally-invariant MGF does not imply the same\footnote{It is possible to extract a weaker bound for the standard MGF, but this loses dimension-dependent factors, which leads to loose bounds. In contrast, our analysis is tight up to a small constant factor.} bound on the standard MGF uniformly over all vectors $v \in \Sd$. But this conversion is of course lossless for rotationally-invariant distributions $\mu$---and this suffices for the purposes of this note. 
\par Let us elaborate: in order to prove concentration inequalities for $\|X\|$ where $X$ is drawn from the (not rotationally-symmetric) distribution $\pieta$, it suffices to prove the same concentration inequalities for $\|\tilde{X}\|$, where $\tilde{X}$ is a randomly rotated version of $X$, namely $\tilde{X} = RX$ where $R$ is a random rotation. Since our analysis in the sequel lets us bound the rotationally-invariant MGF for the law of $\tilde{X}$, this immediately gives concentration inequalities for $\|\tilde{X}\|$ and thus also for $\|X\|$, as desired.

\par Briefly, the upshot of using this rotationally-invariant MGF is that, unlike the standard MGF, it is a function of the norm (see \S\ref{ssec:bessel})---and as we describe in the sequel, this enables exploiting contractivity of the gradient descent step in the Langevin Algorithm update. We also mention that since our goal is to prove concentration inequalities on the norm, of course no information is lost by using the rotational-invariant MGF rather than the standard MGF, since the only difference is averaging over directions.


\subsection{Explicit expression via Bessel functions}\label{ssec:bessel}

\par Of course, by rotational invariance, this Lyapunov function $\Phi_{d,\lambda}(x)$ depends on the point $x$ only through its norm $\|x\|$. That is,
\begin{align}
	\Phi_{d,\lambda}(x) := \phi_{d}(\lambda \|x\|)
	\label{eq:Phi-phi}
\end{align}
where $\phi_d : \R_{\geq 0} \to \R_{\geq 0}$ is the univariate function given by 
\begin{align}
	\phi_d(z) = \E_{v \sim \Sd} \left[ e^{z \langle v, e_1 \rangle} \right] \,.
	\label{eq:phi-def}
\end{align}
The following lemma provides an elegant closed-form expression for this univariate function $\phi_d$ in terms of modified Bessel functions $I_n(\cdot)$ of the first kind and the Gamma function $\Gamma(\cdot)$.

\begin{lemma}[Explicit formula for Lyapunov function]\label{lem:bessel}
	For any dimension $d \geq 2$ and argument $z > 0$,
	\[
		\phi_{d}(z) = \Gamma(\alpha+ 1) \cdot \left(\frac{2}{z} \right)^{\alpha} \cdot I_{\alpha}(z)\,,
	\]
	where $\alpha := (d-2)/2$. (In dimension $d=1$, we simply have $\phi_d(z) = \cosh(z)$.)
\end{lemma}
\begin{proof}
	It is a classical fact (see, e.g.,~\citep{Szego39}) that if $v$ is drawn uniformly at random from the unit sphere $\Sd$, then the induced law $\nu_{d-1}$ of $v_1$ has density
	\begin{align*}
		\frac{d\nu_{d-1}(t)}{dt} = \frac{\Gamma(\alpha + 1)}{\sqrt{\pi} \Gamma(\alpha)} (1-t^2)^{\alpha- 1/2} \mathds{1}_{t \in (-1,1)}\,.
	\end{align*}
	Therefore
	\begin{align*}
			\Phi_{d,\lambda}(x)
			=
			\E_{v \sim \Sd} \left[ e^{\lambda \langle v, x \rangle} \right]
			=
			\E_{v \sim \Sd} \left[ e^{z v_1 } \right]
			= 
			 \frac{\Gamma(\alpha + 1)}{\sqrt{\pi} \Gamma(\alpha)}
			\int_{-1}^1 e^{zt} (1-t^2)^{\alpha- 1/2} dt\,.
	\end{align*}
	We conclude by using the Poisson integral representation 
	\begin{align*}
			I_{\alpha}(z) = \frac{(z/2)^{\alpha}}{\sqrt{\pi}\Gamma(\alpha + 1/2)} \int_{-1}^1 e^{zt} (1-t^2)^{\alpha- 1/2} dt\,.
	\end{align*}
	of the Bessel functions, see, e.g.,~\citep[equation (10.32.2)]{DLMF}.
\end{proof}

\subsection{Properties of the Lyapunov function}

The key reason we use the Lyapunov function $\Phi$ to analyze the dynamics of the Langevin Algorithm is that the Langevin Algorithm's update decomposes into a gradient descent step and a noise convolution step, and $\Phi$ enables tracking the effect of both steps in an elegant manner. 
\par The following lemma describes how $\Phi$ is affected by the noise convolution. While a similar property holds for the standard MGF, the advantage of the rotation-invariant MGF over the standard MGF is that it enables exploiting contractivity of the gradient descent step. However, the effect of the gradient descent step is slightly different depending on whether the potential is just convex or also strongly convex, and thus is treated separately in the corresponding sections below.

\begin{lemma}[Behavior of $\Phi$ under Gaussian convolution]\label{lem:convolve}
	For any dimension $d \in \N$, point $x \in \R^d$, weight $\lambda > 0$, and noise variance $\sigma^2$, 
	\[
	\E_{Z \sim \cN(0, \sigma^2 I_d)} \left[ \Phi_{d,\lambda}\left( x + Z \right) \right] = e^{\lambda^2\sigma^2/2} \cdot \Phi_{d,\lambda}(x)\,.
	\]
\end{lemma}
\begin{proof}
	We compute:
	\begin{align*}
		\E_{Z \sim \cN(0, \sigma^2 I_d)}[\Phi_{d,\lambda}(x+Z)]
		&= 	\E_{Z \sim \cN(0, \sigma^2 I_d)}\E_{v \sim \Sd} \left[ e^{\lambda \langle v, x+Z \rangle} \right]\\
		&= \E_{v \sim \Sd} \left[ e^{\lambda \langle v, x \rangle} \E_{Z \sim \cN(0, \sigma^2 I_d)} [e^{\lambda \langle v, Z\rangle}] \right]\\
		&= \E_{v \sim \Sd} \left[ e^{\lambda \langle v, x \rangle} e^{\lambda^2\sigma^2 \|v\|^2 /2} \right]\\
		&= e^{\lambda^2\sigma^2/2} \cdot \Phi_\lambda(x).
	\end{align*}
	Above, the first and last steps are by definition of $\Phi$; the second step is by Fubini's Theorem; and the third step is by the well-known formula for the moment generating function of the multivariate Gaussian distribution (see, e.g.,~\citep[\S5.8]{grimmett2020probability}).
\end{proof}

The next lemma collects various basic properties of $\Phi$ that we use in our analysis in the sequel. We remark that in the setting of strongly convex potentials, our analysis only makes use of the positivity and monotonicity in item (i) of this lemma; the setting of convex potentials has a more involved analysis and requires the eventual exponential growth in item (ii) of this lemma. 

\begin{lemma}[Properties of rotation-invariant MGF]\label{lem:phi-properties}
	For any dimension $d \in \N$:
	\begin{itemize}
		\item[(i)]  The function $\phi_{d}$ is positive and increasing on $\R_{\geq 0}$. 
		\item[(ii)] There exists some $r_0 := r_0(d)$ such that
		\[
			\phi_d(r + \Delta) \geq e^{\Delta/2} \phi_d(r)\,
		\]
		for all $r \geq r_0$ and $\Delta > 0$.
	\end{itemize}
\end{lemma}
\begin{proof}
	Item (i) is immediate because $\phi_{d}$ is by definition an expectation over functions which are each positive and increasing. For item (ii), observe that Lemma~\ref{lem:bessel} expresses $\phi_d(z)$ as $\Gamma(\alpha + 1)(2/z)^{\alpha} = \Theta(z^{-\alpha})$, which grows inverse polynomially in $z$, times the Bessel function $I_{\alpha}(z)$, which grows exponentially for large arguments $z \gg \sqrt{\alpha} = \Theta(\sqrt{d})$ due to the Hankel expansion $I_{\alpha}(z) \sim e^z (2\pi z)^{-1/2} (1 - \frac{4\alpha^2 - 1}{8z} + \frac{(4 \alpha^2 - 1)(4\alpha^2 - 9)}{2! (8z)^2} - \dots)$~\citep[equation (10.40.1)]{DLMF}. The desired exponential growth immediately follows\footnote{The ``fudge factor'' of $1/2$ in the exponent is used to simply and crudely bound the lower-order polynomial terms, and is irrelevant for our purposes in the sequel.}.
\end{proof}

Although unnecessary for the purposes of this paper, we remark in passing that non-asymptotic bounds on the exponential growth in item (ii) of the above lemma can be computed in several ways. One way is to use more precise non-asymptotic bounds on Bessel functions since, as mentioned above, $\phi_d(z)$ is equal to $I_{\alpha}(z)$ modulo a lower-order polynomial term. Another way is to show that the logarithmic derivative of $\phi_d$ is eventually lower bounded by a constant; we cannot help but mention that an elegant identity for this purpose is that this logarithmic derivative simplifies to $\frac{d}{dz} \log \phi_d(z) = I_{\alpha+1}(z) / I_{\alpha}(z)$, which is the ratio of Bessel functions of different orders.

%% file: sections/sc.tex

\section{Sub-Gaussian concentration for strongly convex potentials}\label{sec:sc}

In this section we show that if the potential $f$ is strongly convex and smooth, then the stationary distribution $\pieta$ of the Langevin Algorithm has sub-Gaussian tails. 	This parallels the classically known fact that in this strongly convex setting, the \emph{unbiased} stationary distribution $\pi \propto e^{-f}$ of the continuous Langevin Diffusion is sub-Gaussian~\citep{ledoux1999concentration}. 

\begin{theorem}[Sub-Gaussianity of $\pieta$ for strongly convex potentials]\label{thm:sc}
	Suppose $f : \R^d \to \R$ is $m$-strongly convex and $M$-smooth for some $0 < m \leq M < \infty$. Consider running the Langevin Algorithm with stepsize $\eta < 2/M$.  Then its stationary distribution $\pieta$ satisfies
	\[
		\Prob_{X \sim \pieta} \left[ \|X - x^*\|
		\geq
		4\sqrt{\frac{\eta}{1-c}} \left(\sqrt{2d} + \sqrt{\log1/\delta}\right)
		\right] \leq \delta\,, \quad \forall \delta \in (0,1)\,,
	\]
	where $x^*$ denotes the unique minimizer of $f$, and $c := \max_{\rho \in \{m,M\}} |1 - \eta \rho|$.
\end{theorem}
	
	The operational interpretation of $c$ is that it is the contraction coefficient corresponding to a gradient descent step (see Lemma~\ref{lem:gd-contractive} below). Note that $c < 1$ because $\eta < 2/M$. In the typical setting of $\eta \leq 1/M$, it simplifies to $c = 1 - \eta m$, whereby Theorem~\ref{thm:sc} simplifies to
	\[
	\Prob_{X \sim \pieta} \left[ \|X - x^*\|
	\geq
	\frac{4}{\sqrt{m}} \left( \sqrt{2d} + \sqrt{\log1/\delta} \right)
	\right] \leq \delta\,, \quad \forall \delta \in (0,1)\,.
	\]

The rest of the section is organized as follows. In \S\ref{ssec:sc-proof} we prove Theorem~\ref{thm:sc} by showing that the rotationally symmetrized version of $\pi$ is sub-Gaussian with variance proxy $2\eta/(1-c)$, and in \S\ref{ssec:sc-tight} we show that this bound is tight in all parameters (up to a constant factor of at most $2$).

\subsection{Proof}\label{ssec:sc-proof}

The proof makes use of the following helper lemma. This fact is well-known in convex optimization because it immediately implies a tight bound on the convergence rate of gradient descent for strongly-convex, smooth objectives. A short proof can be found in, e.g., Appendix A.1 of~\citep{AltTal22mix}.

\begin{lemma}[Contractivity of gradient descent step]\label{lem:gd-contractive}
		Suppose $f$ is $m$-strongly convex and $M$-smooth over $\R^d$ for some $0 \leq m \leq M < \infty$, and consider any stepsize $\eta \leq 2/M$. Then
		\[
				\|x - \eta \nabla f(x) - x^*\| \leq c \|x - x^*\|\,, \quad \forall x \in \R^d\,,
		\]
		where $x^*$ is any minimizer of $f$, and $c := \max_{\rho \in \{m,M\}} |1 - \eta \rho| \leq 1$. 
\end{lemma}

\begin{proof}[Proof of Theorem~\ref{thm:sc}]
	To avoid cumbersome notation, let us assume that $f$ is centered to have minimizer $x^* = 0$ (this is without loss of generality after a translation). We bound the change in the Lyapunov function $\Phi_{d,\lambda}(\cdot)$ from running an iteration of the Langevin Algorithm from some point $x \in \R^d$ to 
	\[
		x' = x - \eta \nabla f(x) + Z\,,
	\]
	where $Z \sim \cN(0,2\eta I_d)$. To analyze this, we disassociate the two parts of this update: the gradient descent step and the noise convolution. 
	\par First, we analyze the gradient descent step:	
		\begin{align*}
		\Phi_{d,\lambda}(x - \eta \nabla f(x))
		= \phi_{d}(\lambda \|x - \eta \nabla f(x)\|)
		\leq \phi_{d}(c \lambda \|x\|)
		\leq \big( \phi_{d}(\lambda \|x\|) \big)^c
		=  \big( \Phi_{d,\lambda}(x ) \big)^c\,.
	\end{align*}
	Above, the first and last steps are by definition of $\phi$ in terms of $\Phi$ (see~\eqref{eq:Phi-phi}); the second step is by contractivity of a gradient descent iteration (\cref{lem:gd-contractive}) and monotonicity of $\phi_d$ (item (i) of Lemma~\ref{lem:phi-properties}); and the third step is by an application of Jensen's inequality $\E[Y^c] \leq \E[Y]^c$ where $Y :=e^{\lambda \langle v, x \rangle}$, which applies since the function $y \mapsto y^c$ is concave for $c \leq 1$.
	
	\par Second, we use Lemma~\ref{lem:convolve} to bound the change in the Lyapunov function from the Gaussian noise convolution:
	\begin{align*}
			\Phi_{d,\lambda}(x') = \Phi_{d,\lambda}(x - \eta \nabla f(x) + Z) = e^{\eta \lambda^2} \Phi_{d,\lambda}(x - \eta \nabla f(x))\,.
	\end{align*}

	\par By combining the above two displays, we conclude that
	\begin{align}
	 	\Phi_{d,\lambda}(x') \leq e^{\eta \lambda^2}\big( \Phi_{d,\lambda}(x ) \big)^c\,.
   \label{eq:recurrence-sc}
	\end{align}
	Now take expectations on both sides, drawing $X \sim \pieta$. Note that by stationarity, $X' \sim \pieta$. Thus
	\[
		\E_{X \sim \pieta}\big[ \Phi_{d,\lambda}(X) \big] \leq e^{\eta \lambda^2} \E_{X \sim \pieta}\big[ \Phi_{d,\lambda}(X) \big]^c\,,
	\]
	where above we have again used Jensen's inequality on the concave function $y \mapsto y^c$. Rearranging and simplifying yields
	\begin{align}
		\E_{X \sim \pieta}\big[ \Phi_{d,\lambda}(X) \big] \leq e^{\eta \lambda^2/(1-c)} \,.
       \label{eq:stationary-sc-prev}
	\end{align}
	By definition of the Lyapunov function $\Phi_{d,\lambda}$, this implies that
	\begin{align}
		\E_{X \sim \pieta} \E_{R \sim \cR} \big[ e^{\lambda \langle u, RX \rangle} \big] \leq e^{\eta \lambda^2/(1-c)}\,, \qquad \forall u \in \Sd\,,
     \label{eq:stationary-sc}
	\end{align}
	where $\cR$ denotes the Haar measure over rotations of $\R^d$. Define $\tilde{\pi}_{\eta}$ to be the law of $RX$ (in words, $\tilde{\pi}_{\eta}$ is a rotationally symmetrized version of $\pieta$). Then by definition of sub-Gaussianity, the above display establishes that $\tilde{\pi}_{\eta}$ is sub-Gaussian with variance proxy $2\eta/(1-c)$. By a standard concentration inequality on the norm of a sub-Gaussian vector (see e.g.,~\citep[Theorem 1.19]{rigollet2015high}),
	\[
			\Prob_{X \sim \tpieta} \left[ \|X\|
		\geq
		4\sqrt{\frac{\eta}{1-c}} \left(\sqrt{2d} + \sqrt{\log1/\delta}\right)
		\right] \leq \delta\,, \quad \forall \delta \in (0,1)\,.
	\]
	Since the Euclidean norm is invariant under rotations, this concentration inequality remains true if $\tilde{\pi}_{\eta}$ is replaced by $\pieta$ (cf., the discussion in \S\ref{ssec:mgf}). This concludes the proof.
\end{proof}

\subsection{Tightness}\label{ssec:sc-tight}

Here we observe that, modulo a constant factor of at most $2$, the sub-Gaussianity we proved is tight in all parameters. It is simplest to explain this tightness in terms of the sub-Gaussian parameters since the final concentration inequality in Theorem~\ref{thm:sc} was proved as a direct consequence of this. To this end, recall that in the proof of Theorem~\ref{thm:sc} above, we proved that the rotationally symmetric version $\tpieta$ of $\pieta$ is sub-Gaussian with variance proxy $2\eta/(1-c)$. Following is a simple construction which matches this upper bound. This calculation is similar to~\citep[Example 4]{VempalaW19}.

\begin{example}[Tightness of sub-Gaussanity parameters]\label{ex:sc-tight}
	Consider running the Langevin Algorithm on the univariate quadratic potential $f(x) = \frac{\rho}{2}x^2$, where the curvature $\rho$ is chosen so as to maximize the contraction coefficient $c := \max_{\rho \in \{m,M\}} |1 - \eta \rho|$. Then, for any stepsize $\eta > 0$, the Langevin Algorithm has stationary distribution
	\[
			\pieta = \cN\left( 0, 2\eta \frac{1}{1-c^2} \right)
	\]
	when it is initialized to $X_0 = 0$. (This fact is proved in~\citep[\S4.2]{AltTal22mix} and follows by simply unraveling the definition of a Langevin Algorithm update, composing $t$ times, and taking $t \to \infty$.)
	Note that $\tpieta = \pieta$ in this example since $\pieta$ is already rotationally symmetric.
	Moreover, since $1-c^2 = (1-c)(1+c)$ and since also $(1+c) \in (1,2)$, this construction matches our aforementioned upper bound on the sub-Gaussian parameter of $\tpieta$ up to a constant factor of at most $2$.
\end{example}

%% file: sections/convex.tex

\section{Sub-exponential concentration for convex potentials}\label{sec:convex}

In this section we show that if the potential $f$ is convex and smooth, then the stationary distribution $\pieta$ of the Langevin Algorithm has sub-exponential tails. This parallels the classically known fact that in this convex setting, the \emph{unbiased} stationary distribution $\pi \propto e^{-f}$ of the continuous Langevin Diffusion is sub-exponential (see Lemma~\ref{lem:pi-convex} below).

\begin{theorem}[Sub-exponentiality of $\pieta$ for convex potentials]\label{thm:convex}
    Consider running the Langevin Algorithm on a convex, $M$-smooth potential $f : \R^d \to \R$ with stepsize $\eta \leq 1/M$. Then the stationary distribution $\pieta$ of the Langevin Algorithm has sub-exponential concentration; i.e., there exists
    $A,C,R > 0$ and a minimizer $x^*$ of $f$ such that
    \[
        \Prob_{X \sim \pieta} \left[ \norm{X - x^*} \geq R + C \log(A/\delta) \right] \leq \delta\,, \quad \forall \delta \in (0,1)\,.
    \] 
\end{theorem}

Note that in this sub-exponential concentration inequality, the parameters $A$, $C$, and $R$ depend on $f$ and $\eta$ (this dependence is made explicit in the proof). A dependence on $f$ is inevitable since even the unbiased distribution $\pi$ has sub-exponential parameters which depend on $f$. As a simple concrete example, consider the potential $f(x) = d(x, B_R)$ where $B_R$ is the Euclidean ball of radius $R$ centered at the origin, and $d(x,B_R)$ is the Euclidean distance to the ball. The distribution $\pi$ then concentrates no better than the uniform distribution on the Euclidean ball of radius $R$, and a randomly drawn point $X$ from $B_R$ has norm $\Theta(R)$ with high probability. See also the tightness discussion in \S\ref{ssec:convex-tight} for a fully worked-out example.

The rest of the section is organized as follows. In \S\ref{ssec:sc-proof} we prove Theorem~\ref{thm:sc} by showing that the rotationally symmetrized version of $\pi$ is sub-Gaussian with variance proxy $2\eta/(1-c)$, and in \S\ref{ssec:sc-tight} we show that this bound is tight in all parameters (up to a constant factor of at most $2$).

The rest of the section is organized as follows. In \S\ref{ssec:convex-proof} we prove Theorem~\ref{thm:convex}, and in \S\ref{ssec:convex-tight} we discuss tightness of this bound.

\subsection{Proof}\label{ssec:convex-proof}

We begin by recalling the classical fact from functional analysis that every log-concave distribution is sub-exponential. For a proof, see e.g.,~\citep[Lemma 10.6.1]{MilmanBook}.

\begin{lemma}[Log-concavity implies sub-exponentiality]\label{lem:pi-convex}
	If $\pi$ is a log-concave probability distribution, then there exist 
	$a,b > 0$ such that $\pi(x) \leq ae^{-b\|x\|}$ for all $x \in \R^d$. 
\end{lemma}

To apply this lemma, it is convenient to first assume without loss of generality that $f$ is translated and shifted so that $0$ is a minimizer of $f$ with value $f(0) = 0$. Then $f(x) = \log ( \pi(0) / \pi(x) )$. By Lemma~\ref{lem:pi-convex}, we conclude that
\begin{align}
	f(x) \geq -\alpha + \beta\|x\|\,, \quad \forall x \in \R^d\,,
	\label{eq:convex-superlinear}
\end{align}
for some $\alpha,\beta \in \R$. (Here, $\alpha$ plays the role of $\log a - \log \pi(0)$ and $\beta$ plays the role of $b$, where $a$ and $b$ are the quantities from Lemma~\ref{lem:pi-convex}.)
We exploit this super-linear growth~\eqref{eq:convex-superlinear} of $f$ to show that a gradient descent update on $f$ makes significant progress towards the minimizer $x^* =0$ when the current iterate has sufficiently large norm.

\begin{lemma}[Gradient descent updates make significant progress for super-linear objectives]\label{lem:gd-superlinear}
	Suppose $f : \R^d \to \R$ is convex, $M$-smooth, achieves its minimum value at $f(0) = 0$, and satisfies~\eqref{eq:convex-superlinear}. Then there exists $r_1 := r_1(f)$ such that for any stepsize $\eta \leq 1/M$, it holds that
	\[
	\|x - \eta \nabla f(x)\| \leq \|x \| - \frac{\eta \beta}{4} \,, \quad \forall \|x\| \geq r_1\,.
	\]
\end{lemma}
\begin{proof}
	We bound
	\begin{align*}
		\|x - \eta \nabla f(x)\| & = \sqrt{\|x - \eta \nabla f(x)\| ^2}
		\\ &= \sqrt{\|x\|^2 - 2\eta \langle \nabla f(x), x \rangle + \eta^2\|\nabla f(x)\|^2}\\
		&\leq  \sqrt{\|x\|^2 -  2\eta \langle \nabla f(x), x \rangle + \eta^2 M \langle \nabla f(x), x \rangle} \\
		&\leq \sqrt{\|x\|^2 - \eta \langle \nabla f(x), x \rangle}\\
		&\leq \|x\| - \eta \langle\nabla f(x), x \rangle / (2 \|x\| ) \\
		&\leq \|x\| - \eta f(x)/(2\|x\|) \\
		&\leq \|x\| - \eta \beta/2  + \eta \alpha /(2\|x\|) \,.
	\end{align*}
	Above, the second step expands the square. The third step uses the folklore fact from optimization (see, e.g.,~\citep[Equation (2.1.8)]{nesterov2003introductory}) that $M$-smoothness implies $\langle \nabla f(x) - \nabla f(y), x - y\rangle \geq \frac{1}{M} \|\nabla f(x) - \nabla f(y)\|^2$, which we use here for $y=0$. The fourth step uses the assumption that $\eta \leq 1/M$. The fifth step adds a non-negative quantity to complete the square. The penultimate step uses convexity and the assumption $f(0) = 0$ to bound $\langle \nabla f(x), x \rangle \geq f(x) - f(0) = f(x)$. The final step uses the super-linear growth condition~\eqref{eq:convex-superlinear}.  We conclude that if $\|x\| \geq 2\alpha /\beta$, then $\|x - \eta \nabla f(x)\| \leq \|x\| - \eta \beta/4$. 
\end{proof}

\begin{proof}[Proof of Theorem~\ref{thm:convex}]
	As in the proof of Theorem~\ref{thm:sc}, we bound the change in the Lyapunov function $\Phi_{d,\lambda}(\cdot)$ from running an iteration of the Langevin Algorithm from some point $x \in \R^d$ to 
	\[
	x' = x - \eta \nabla f(x) + Z\,,
	\]
	where $Z \sim \cN(0,2\eta I_d)$. To analyze this, we disassociate the two parts of this update: the gradient descent step and the noise convolution.
	\par First, we analyze the gradient descent step. We consider two cases depending on whether the norm $\|x\|$ of the current iterate is larger than $R := \max(r_0/\lambda + \eta \beta /4, r_1 )$, where $r_0$ is the quantity from Lemma~\ref{lem:phi-properties}, and $r_1$ is the quantity from Lemma~\ref{lem:gd-superlinear}. (We will set $\lambda = \beta/16$.)
	
	\begin{itemize}
		\item \underline{Case 1: $\|x\| \leq R$.} Then by the fact that a gradient descent step is non-expansive (Lemma~\ref{lem:gd-contractive} with $m=0$), we know that $\|x - \eta \nabla f(x)\| \leq \|x\| \leq R$. Thus
		\begin{align*}
			\Phi_{d,\lambda}(x - \eta \nabla f(x)) = \phi_{d}(\lambda \|x - \eta \nabla f(x)\|) \leq \phi_{d}( \lambda R)\,.
		\end{align*}
		where above the first step is by definition of $\phi$ in terms of $\Phi$ (see~\eqref{eq:Phi-phi}), and the second step is by monotonicity of $\phi_{d}$ (item (i) of Lemma~\ref{lem:phi-properties}).
		\item \underline{Case 2: $\|x\| > R$.} Then $\|x\| \geq r_1$, so $\|x - \eta \nabla f(x)\| \leq \|x\| - \eta \beta /4$ by Lemma~\ref{lem:gd-superlinear}. Thus
		\begin{align*}
			\Phi_{d,\lambda}(x - \eta \nabla f(x)) =  \phi_{d}(\lambda \|x - \eta \nabla f(x)\|) \leq \phi_{d}(\lambda (\|x\| - \eta \beta/4)) \leq e^{-\eta \beta \lambda/8} \phi_{d}(\lambda \|x\|) = e^{-\eta \beta \lambda /8} \Phi_{d,\lambda}(x) \,,
		\end{align*}
		where above the first and last steps are by definition of $\phi$ in terms of $\Phi$ (see~\eqref{eq:Phi-phi}), the second step is by monotonicity of $\phi_{d}$ (item (i) of Lemma~\ref{lem:phi-properties}), and the third step is due to the super-exponential growth of $\phi_{d}$ above $\lambda (R - \eta \beta /4) \geq r_0$ (item (ii) of Lemma~\ref{lem:phi-properties}).
	\end{itemize}

	\par Now, since $\phi$ is non-negative (item (i) of Lemma~\ref{lem:phi-properties}), we can crudely combine these two cases to conclude that
	\begin{align*}
			\Phi_{d,\lambda}(x - \eta \nabla f(x)) \leq \phi_{d}(\lambda R) + e^{-\eta \beta \lambda / 8} \Phi_{d,\lambda}(x)\,.
	\end{align*}

	\par Now use Lemma~\ref{lem:convolve} to bound the change in the Lyapunov function from the Gaussian noise convolution:
	\begin{align}
		\Phi_{d,\lambda}(x') = \Phi_{d,\lambda}(x - \eta \nabla f(x) + Z) = e^{\eta \lambda^2} \Phi_{d,\lambda}(x -  \eta \nabla f(x))
        \label{eq:recurrence-convex}
	\end{align}
	
	\par By combining the above two displays, we conclude that
	\[
	\Phi_{d,\lambda}(x') \leq e^{\eta \lambda^2} \left[ \phi_{d}(\lambda R) + e^{-\eta \beta \lambda / 8} \Phi_{d,\lambda}(x) \right] \,.
	\]
	Now take expectations on both sides, drawing $X \sim \pieta$. Note that by stationarity, $X' \sim \pieta$. Thus
	\[
	\E_{X \sim \pieta}\big[ \Phi_{d,\lambda}(X) \big] \leq e^{\eta \lambda^2}\left(  \phi_{d}(\lambda R) + e^{-\eta \beta \lambda/ 8}  \E_{X \sim \pieta} \left[ \Phi_{d,\lambda}(X) \right]  \right) \,.
	\]
	Re-arranging and simplifying yields
	\[
		\E_{X \sim \pieta}\big[ \Phi_{d,\lambda}(X) \big]  \leq \frac{e^{\eta \lambda^2}}{1 - e^{\eta \lambda (\lambda - \beta/8)}} \phi_d(\lambda R)\,.
	\]
	Setting $\lambda = \beta/16$ yields
	\begin{align}
	\E_{X \sim \pieta}\big[ \Phi_{d,\lambda}(X) \big] \leq A \phi_{d}(\lambda R)\,,
	\label{eq:convex-Phi-bound}
	\end{align}
	where for notational shorthand we have defined $A := e^{\eta \beta^2 /256} / (1 - e^{-\eta \beta^2 /256}) > 1$. 
	\par From this, we argue a concentration inequality in a way analogous to Chernoff bounds, except using our rotationally-invariant version of the moment generating function. Specifically, bound
	\begin{align*}
			\Prob_{X \sim \pieta} \left[ \|X\| \geq R + 2/\lambda \log (A/\delta) \right]
			&=  \Prob_{X \sim \pieta} \left[ \phi_{d}(\lambda \|X\|) \geq \phi_{d}\left( \lambda R + 2 \log (A/\delta)  \right) \right] \\
			&\leq \frac{\E_{X \sim \pieta} \left[ \phi_{d}(\lambda \|X\|) \right] }{ \phi_{d}\left( \lambda R + 2 \log (A/\delta) \right) } \\
			&\leq  \frac{ A \phi_{d}(\lambda R) }{e^{\log (A/\delta)} \phi_{d}\left( \lambda R \right) } \\
			&= \delta\,.
	\end{align*}
	Above, the first step is by monotonicity of $\phi_{d}$ (item (i) of Lemma~\ref{lem:phi-properties}), the second step is by Markov's inequality which is applicable since $\phi_{d}$ is non-negative (item (i) of Lemma~\ref{lem:phi-properties}), and the third step is by~\eqref{eq:convex-Phi-bound} and the super-exponential growth of $\phi_{d}$ above $\lambda R \geq r_0$ (item (ii) of Lemma~\ref{lem:phi-properties}).
\end{proof}

\subsection{Tightness}\label{ssec:convex-tight}

Here we observe that, modulo a constant factor, the sub-exponential concentration bound that we proved for $\pieta$ matches the concentration for $\pi$---both in terms of when the sub-exponentiality threshold occurs, and also in terms of the sub-exponential decay rate beyond this threshold. Although no techniques currently exist for extending these lower bounds from $\pi$ to $\pieta$ (this is an interesting open question), we conjecture that this tightness extends from $\pi$ to $\pieta$ since we suspect that $\pieta$ should not have better concentration than $\pi$.

\begin{example}[Tightness of sub-exponentiality parameters]\label{ex:convex-tight}
    Consider running the Langevin Algorithm where the potential $f$ is the univariate Huber function
	\[
		f(x) = \begin{cases}
				\beta x^2 /2 & |x| \leq 1 \\
				\beta|x|-\alpha & |x| \geq 1
		\end{cases}
	\]
	where $\alpha := \beta/2$ is set so that $f$ and $f'$ are continuous at the breakpoints. Intuitively, modulo a translation and a quadratic smoothing around $0$, $f$ is essentially the absolute value function times $\beta$, and therefore $\pi$ is essentially the exponential distribution with parameter $\beta$. It is straightforward to check that $f$ is convex and $M$-smooth with $M=\beta$, and thus satisfies the conditions in Theorem~\ref{thm:convex}. In this example, the quantities in the proof of Theorem~\ref{thm:convex} are as follows. To simplify the asymptotics, suppose that $\beta \in (0,1)$.
	\begin{itemize}
		\item The quantities $\alpha$ and $\beta$ in the super-linear growth inequality~\eqref{eq:convex-superlinear} exactly match the $\alpha$ and $\beta$ in the construction of this function $f$.
		\item The quantity $r_0$ from Lemma~\ref{lem:phi-properties} can be set to $r_0 = \log(3)/2 \approx 0.5493$ because in dimension $d=1$, the Lyapunov function $\phi_d(r) = \cosh(r)$, and it holds\footnote{Proof: By Jensen's inequality, $\frac{3}{4} \exp(\Delta) + \frac{1}{4} \exp(-\Delta) \geq \exp(\Delta/2)$. Re-arranging gives $e^{2r} \geq 3 \geq \frac{\exp(\Delta/2) - \exp(-\Delta)}{\exp(\Delta) - \exp(\Delta/2)}$ for any $r \geq \log(3)/2$. Re-arranging further establishes the desired inequality $\cosh(r + \Delta) \geq \exp(\Delta/2) \cosh(r)$.} that $\cosh(r + \Delta) \geq \exp(\Delta/2) \cosh(r)$ for all $r \geq \log(3)/2$ and all $\Delta \geq 0$.
		\item The quantity $r_1$ from Lemma~\ref{lem:gd-superlinear} is $r_1 = \frac{2\alpha}{\beta} = 1$. 
		\item The quantity $R = \max(\frac{r_0}{\lambda} + \frac{\eta \beta}{4}, r_1) \asymp \max( \frac{1}{\beta} + \frac{\eta\beta}{1}, 1) \asymp \frac{1}{\beta}$ since $\eta \leq \frac{1}{\beta}$.
	\end{itemize}
	Therefore, for this potential $f$, Theorem~\ref{thm:convex} establishes sub-exponential tails that decay at rate $2\lambda^{-1}\logdel \asymp\beta^{-1}\logdel$ beyond norms of size $R + 2\lambda^{-1} \log A \asymp \beta^{-1}$. Similarly, $\pi \propto e^{-f}$ has sub-exponential tails that decay at a rate of order $\beta^{-1} \logdel$ beyond norms of order $\beta^{-1}$, because $\pi$ has tails that are similar to that of an exponential random variable with parameter $\beta$. 
	We conclude that for all stepsizes $\eta \leq 1/M$ that are not exponentially small in the problem parameters\footnote{This is a mild assumption, since in all relevant parameter regimes for sampling, $\eta$ is only polynomially small; and moreover, for exponentially small $\eta$, sub-exponential concentration of $\pieta$ is immediate from combining any bias bound between $\pi$ and $\pieta$ (which depend polynomially on $\eta$) with the classical result that $\pi$ is sub-exponential (Lemma~\ref{lem:pi-convex}).}, the concentration inequality in Theorem~\ref{thm:convex} matches that of $\pi$---both in terms of when the sub-exponentiality threshold occurs, as well as the rate of sub-exponential decay beyond this threshold.
\end{example}

%% file: sections/extensions.tex

\section{Further extensions}\label{sec:extensions}

\subsection{Concentration along the trajectory}\label{ssec:extensions:trajectory}

In this subsection we mention that the concentration we showed for $\pi_{\eta}$ extends to concentration of the $t$-th iterate $X_t$ of the Langevin Algorithm, for every $t$. This is a straightforward extension of our proofs for $\pi_{\eta}$, recovers that result in the limit $t \to \infty$, and provides slightly tighter concentration bounds for every finite $t$. Moreover, it applies to both the settings of strongly and non-strongly convex potentials, which we studied in \S\ref{sec:sc} and \S\ref{sec:convex}, respectively. We state the results for both settings below. These results are most simply stated when the algorithm is initialized to a minimizer $x^*$ of $f$, which is a standard pre-processing for sampling algorithms (see, e.g.,~\citep{Chewi22}) and can be done efficiently using the same access to first-order queries (e.g., via gradient descent). 

\begin{theorem}[Sub-Gaussianity of Langevin iterates for strongly convex potentials]\label{thm:sc-trajectory}
    Consider the setting of Theorem~\ref{thm:sc} and suppose the Langevin Algorithm is initialized at $x^*$. Then the concentration shown for $\pi_{\eta}$ in Theorem~\ref{thm:sc} also applies to every iterate of the Langevin Algorithm.
\end{theorem}
\begin{proof}
    The proof is a simple extension of the proof of Theorem~\ref{thm:sc}. Recall the key step~\eqref{eq:recurrence-sc} in that proof. By applying this inductively to iterates of the Langevin Algorithm (i.e., applying this inequality with $X_t$ as $x$), we obtain the following recursion for the rotation-invariant MGF:
    \begin{align*}
        \Phi_{d,\lambda}(X_{t+1}) \leq e^{\eta \lambda^2}\big( \Phi_{d,\lambda}(X_t ) \big)^c\,.
    \end{align*}
    Moreover, note that $\Phi_{d,\lambda}(X_0) = 1$ since the initialization is $X_0 = x^* = 0$. Unraveling this recursion, it follows that
    \begin{align*}
        \Phi_{d,\lambda}(X_{t}) \leq 
        e^{\eta \lambda^2 \sum_{s=0}^{t-1} c^s}
        = e^{\eta \lambda^2 (1 - c^t)/(1-c)}
        \,,
    \end{align*}
    for all iterations $t \geq 0$. This recovers (and improves for any finite $t$) the upper bound in~\eqref{eq:stationary-sc-prev} which is for the stationary distribution (a.k.a., the limit $t \to \infty$). The rest of the proof then follows identically as done there (simply use sub-Gaussian concentration).
\end{proof}


\begin{theorem}[Sub-exponentiality of Langevin iterates for convex potentials]\label{thm:convex-trajectory}
    Consider the setting of Theorem~\ref{thm:convex} and suppose the Langevin Algorithm is initialized at any minimizer of $f$. Then the concentration shown for $\pi_{\eta}$ in Theorem~\ref{thm:convex} also applies to every iterate of the Langevin Algorithm.
\end{theorem}
\begin{proof}
The proof is a simple extension of the proof of Theorem~\ref{thm:convex} (mirroring the extension described above for the strongly convex case). Recall the key step~\eqref{eq:recurrence-convex} in the proof of Theorem~\ref{thm:convex}. By applying this inductively to iterates of the Langevin Algorithm (i.e., applying this inequality with $X_t$ as $x$) and setting $\lambda = \beta / 16$ as done in the proof of Theorem~\ref{thm:convex}, we obtain the following recursion for the rotation-invariant MGF:
\begin{align*}
    \Phi_{d,\lambda}(X_{t+1}) \leq \rho^{-1} \phi_{d}(\lambda R) + \rho \Phi_{d,\lambda}(X_t) \,,
\end{align*}
where for shorthand we denote $\rho := e^{-\eta \beta^2 / 256}$. Moreover, note that $\Phi_{d,\lambda}(x_0) = 1$ since the initialization is $x_0 = x^* = 0$. Unraveling this recursion, it follows that
\begin{align}
    \Phi_{d,\lambda}(X_{t}) \leq 
    \frac{1 - \rho^t}{\rho(1 - \rho)} \phi_d(\lambda R) + \rho^t
    \leq \frac{1-\rho^{t+1}}{\rho(1 - \rho)} \phi_d(\lambda R)
    \,,
    \label{eq:convex-unraveled}
\end{align}
for all iterations $t \geq 0$. Above, the last step is by a crude bound that uses the fact $\rho \leq 1$ and $1 = \phi_d(0) \leq \phi_d(\lambda R)$. This recovers (and improves for any finite $t$) the upper bound in~\eqref{eq:convex-Phi-bound} which is for the stationary distribution (a.k.a., the limit $t \to \infty$). The rest of the proof then follows identically as done there (i.e., by using a Chernoff-type argument).
\end{proof}
    
\subsection{Using inexact gradients}\label{ssec:extensions:inexact}

In this subsection we mention that the results in this paper apply to inexact implementations of the Langevin Algorithm which use sub-Gaussian estimates of the true gradients. Specifically, here we consider implementations of the form
\begin{align}
    X_{t+1} = X_t - \eta G_t + Z_t\,,
    \label{eq:LA-inexact}
\end{align}
where $\xi_t := G_t - \nabla f(X_t)$ is sub-Gaussian with some variance proxy $\sigma^2$ and, conditional on $X_t$, is independent of all other randomness. 
\par This extension applies to all settings in this paper---both the settings of strongly and non-strongly convex potentials, and both the concentration of the stationary distribution $\pi_{\eta}$ as well as the entire trajectory of the Langevin Algorithm. For brevity, we state these results for the trajectory; the same bounds then hold for the stationary distribution by taking the limit $t \to \infty$. 
In what follows, we use the standard convention that $\xi$ being sub-Gaussian with variance proxy $\sigma^2$ means $\E[\xi] = 0$ and $\E[e^{s \langle v, \xi \rangle}] \leq e^{s^2 \sigma^2/2}$ for any unit vector $v$. 

\begin{theorem}[Sub-Gaussianity of inexact Langevin for strongly convex potentials]\label{thm:sc-inexact}
    Consider the setting of Theorem~\ref{thm:sc-trajectory}, except now with the Langevin Algorithm implemented inexactly as in~\eqref{eq:LA-inexact}, where $\xi_t := G_t - \nabla f(X_t)$ is sub-Gaussian with variance proxy $\sigma^2$. Then for each $t$, the $t$-th iterate of this algorithm satisfies
    \[
    \Prob \left[ \|X_t - x^*\|
    \geq
    4\sqrt{\frac{\eta + \eta^2 \sigma^2/2}{1-c}} \left(\sqrt{2d} + \sqrt{\log1/\delta}\right)
    \right] \leq \delta\,, \quad \forall \delta \in (0,1)\,,
\]
\end{theorem}
\begin{proof}
    By a nearly identical argument, the key recursion~\eqref{eq:recurrence-sc} becomes
    \begin{align*}
            \Phi_{d,\lambda}(x') = e^{\eta \lambda^2 + \eta^2 \lambda^2 \sigma^2/2} \; \Phi_{d,\lambda}(x - \eta \nabla f(x))\,,
    \end{align*}
    where the only difference is the extra factor of $\Phi_{\lambda,d}(\eta (G_t - \nabla f(X_t))) = \Phi_{\lambda,d}(\eta \xi_t) = e^{\eta^2 \lambda^2 \sigma^2/2}$ arising from the inexact gradients and the definition of sub-Gaussianity. The rest of the proof is then identical, with the sub-Gaussian proxy changing from $\frac{2\eta}{1-c}$ to $\frac{2\eta + \eta^2 \sigma^2/2}{1-c}$ due to this extra term.
\end{proof}

\begin{theorem}[Sub-exponentiality of inexact Langevin for convex potentials]\label{thm:sc-inexact}
    Consider the setting of Theorem~\ref{thm:convex-trajectory}, except now with the Langevin Algorithm implemented inexactly as in~\eqref{eq:LA-inexact}, where $\xi_t := G_t - \nabla f(X_t)$ is sub-Gaussian with variance proxy $\sigma^2$. Then there exist $A,C,R > 0$ such that for each $t$, the $t$-th iterate of this algorithm satisfies
    \[
        \Prob_{X \sim \pieta} \left[ \norm{X_t - x^*} \geq R + C \log(A/\delta) \right] \leq \delta\,, \quad \forall \delta \in (0,1)\,.
    \]
\end{theorem}
\begin{proof}
    By a nearly identical argument, the key recursion~\eqref{eq:recurrence-convex} becomes
    \begin{align*}
           \Phi_{d,\lambda}(x') \leq e^{\eta \lambda^2 + \eta^2 \lambda^2 \sigma^2 / 2} \left[ \phi_{d}(\lambda R) + e^{-\eta \beta \lambda / 8} \Phi_{d,\lambda}(x) \right]\,,
    \end{align*}
    due to the sub-Gaussian term, as explained in the above proof for the strongly convex case. The rest of the proof is then identical to that of Theorem~\ref{thm:convex-trajectory}, except with $\lambda$ now set to $\beta/(16 + 8\eta \sigma^2)$, so that $\rho$ now becomes $e^{ - \eta \beta \lambda / 16} = e^{-\eta \beta^2 / (256 + 128 \eta \sigma^2)}$, which just changes the definition of $A = \tfrac{1}{\rho(1-\rho)}$.
\end{proof}